\documentclass[12pt,a4paper,oneside,onecolumn]{book}
\usepackage[utf8x]{inputenc}

\usepackage{silence}
\usepackage{fancyhdr}
\fancypagestyle{plain}{\pagestyle{fancy}}
\usepackage[labelformat=simple]{subcaption}

\usepackage{textcomp} 
\usepackage{gensymb} 
\usepackage{mathtools} 
\usepackage{amssymb} 
\usepackage[shortlabels]{enumitem}
\usepackage[multiple]{footmisc}
\usepackage{tikz} 
\usepackage{multirow} 
\usepackage{array} 
\usepackage{wrapfig}
\usepackage{graphicx}

\usepackage[amsmath, thmmarks]{ntheorem}

\newtheorem{definition}{Definition}
\newtheorem{theorem}{Theorem}

\newtheorem{proposition}[theorem]{Proposition}

\theoremstyle{nonumberplain}
\theoremheaderfont{\itshape}
\theorembodyfont{\normalfont}
\theoremseparator{.}
\theoremsymbol{\ensuremath{\square}}
\newtheorem{proof}{Proof}

\theoremstyle{plain}
\theoremheaderfont{\normalfont\bfseries}
\theorembodyfont{\itshape}
\theoremseparator{}
\theoremsymbol{}

\usepackage{chngcntr}
\counterwithin{theorem}{chapter}

\DeclareMathOperator*{\argmax}{arg\,max}
\usepackage[ruled,vlined,algochapter]{algorithm2e}
\usepackage{csquotes}
\newcommand{\Eta}{\mathrm{H}}

\usepackage[numbers,sort,compress,square]{natbib}

\usepackage[nottoc]{tocbibind}

\makeatletter
\def\NAT@spacechar{~}
\makeatother

\usepackage{tikz}
\usetikzlibrary{bayesnet}
\usetikzlibrary{arrows}
\usetikzlibrary{decorations.pathmorphing}
\usetikzlibrary{matrix,positioning}

\usepackage{float} 

\usepackage{graphicx}
\graphicspath{ {content/images/} }

\usepackage{titling}
\usepackage[titletoc]{appendix}

\usepackage[capitalise]{cleveref}
\makeatletter
\def\@@number#1{#1}
\makeatother

\Crefname{observation}{Observation}{Observations}
\Crefname{assumption}{Assumption}{Assumptions}

\usepackage{setspace}
\newcommand{\thesistitle}{\textbf{Hierarchical Time Series Forecasting with Bayesian Modeling}}
\newcommand{\thesisauthorname}{\textbf{Gal Elgavish}}
\newcommand{\thesissupervisername}{\textbf{Prof. Eyal Shimony and Dr. David Tolpin}}
\newcommand{\thesismonth}{\textbf{August}}
\newcommand{\thesisyear}{\textbf{2023}}

\begin{document}

\begin{titlepage}
    \begin{center}
        \vspace*{1cm}
        
        \includegraphics[width=0.1\textwidth]{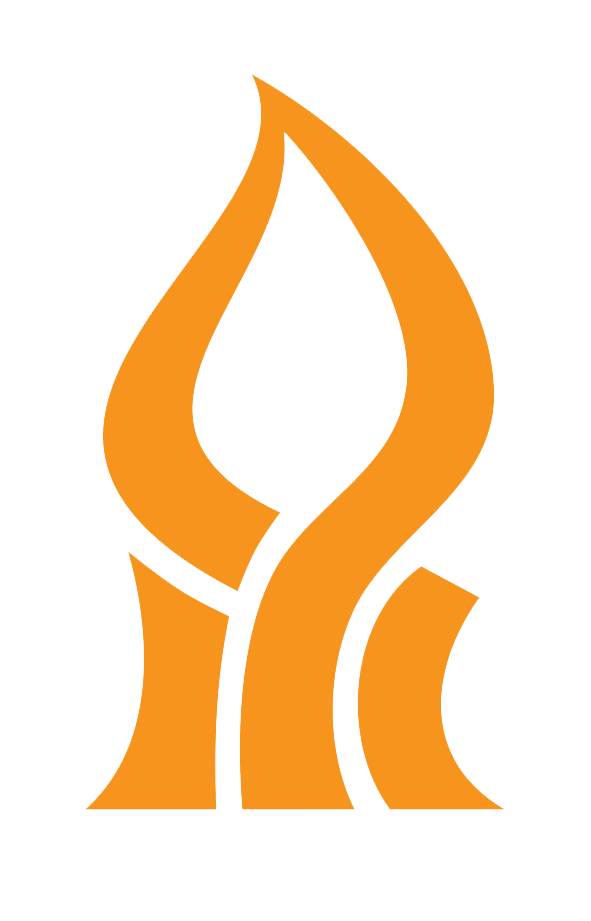}\\
        Ben-Gurion University of the Negev\\
        The Faculty of Natural Sciences\\
        The Department of Computer Science
        
        \vspace{2cm}
        
        {\Large \thesistitle}
        
        \vspace{1cm}
        
        Thesis submitted in partial fulfillment of the requirements\\for the Master of Sciences degree
        
        \vspace{1cm}
        
        \thesisauthorname
        
        \vspace{1cm}
        
        Under the supervision of \thesissupervisername
        
        \vfill
        
        \thesismonth \quad \thesisyear

    \end{center}
\end{titlepage}

\frontmatter

\thispagestyle{plain}

\begin{center}
    \Large
    \thesistitle

    \vspace{0.4cm}
    \large
    \thesisauthorname
       
    \vspace{0.4cm}
    \large
    Master of Sciences Thesis 
      
    \vspace{0.2cm}
    \large
    Ben-Gurion University of the Negev 
    
    \vspace{0.2cm}
    \large
    \thesisyear
    
    \vspace{0.9cm}
    \Large
    \textbf{Abstract}
\end{center}

We encounter time series data in many domains such as finance, physics, business, and weather.
One of the main tasks of time series analysis, one that helps to take informed decisions under uncertainty, is forecasting.
Time series are often hierarchically structured, e.g., a company sales might be broken down into different regions, and each region into different stores.
In some cases the number of series in the hierarchy is too big to fit in a single model to produce forecasts in relevant time, and a decentralized approach is beneficial.

One way to do this is to train independent forecasting models for each series and for some summary statistics series implied by the hierarchy (e.g. the sum of all series) and to pass those models to a \textit{reconciliation} algorithm to improve
those forecasts by sharing information between
the series.

In this work we focus on the \textit{reconciliation} step, and propose a method to do so from a Bayesian perspective --- \textit{Bayesian forecast reconciliation}.
We also define the common case of \textit{linear Gaussian reconciliation}, where the forecasts are Gaussian and the hierarchy has linear structure, and show that we can compute \textit{reconciliation} in closed form.
We evaluate these methods on synthetic and real data sets, and compare them to other work in this field.

\chapter*{Acknowledgements}

I would like to thank Dr. David Tolpin who not only supervised me in this research with great patience and attention, but also taught me how to learn new stuff and how to approach problems.

I will also like to thank Prof. Eyal Shimony who made this research possible by giving great advices, comments and challenged me to improve my work.

\tableofcontents
\clearpage

\listoffigures
\clearpage

\listoftables

\mainmatter

\chapter{Introduction}
\label{ch:intro}

\section*{}

Time series data appears in many domains, such as: finance, physics, business, weather, and basically every field that has temporal measurements.
In many of these fields decisions must be taken today based on what will happen in the future, which makes \textit{Time Series Forecasting} an important task in helping stakeholders to make informed decisions under uncertainty.

Different time series forecasting problems might have different characteristics:
\begin{itemize}
    \item Dimensionality of the time series are either univariate or multivariate.
    \item Time stamps may be regularly or irregularly sampled.
    \item The forecast horizon may be either short term (\textit{nowcasting}~\cite{wiki:nowcasting}) or long term.
    \item Time series might have seasonality, trend, auto-correlation as well as other features.
    \item Forecasts might be point-forecasts or probabilistic forecasts.
\end{itemize}

Methods for time series forecasting treat the series as a stochastic process.
Classical methods of forecasting include auto-regressive integrated moving average (\textit{ARIMA}) models~\cite{box-jenkins} and \textit{Kalman filtering}~\cite{kalman-filter}, while more modern methods include \textit{Gaussian processes}~\cite{gp-for-ml, duvenaud-thesis-2014} and \textit{deep neural networks}~\cite{DeepAR, gp-var}.
We will elaborate on some techniques in Chapter~\ref{ch:background}.

\section*{}

There are two types of multivariate time series: heterogeneous and homogeneous.
An example of a heterogeneous time series is, for example, medical monitoring --- blood pressure, heart rate, saturation etc. The time series are related but they are of different natures, and the number of components is small.
Examples of a homogeneous time series are stock prices, exchange rates, precipitation levels by area and so on. All of the components have the same `nature', and the number of components can be very large.
While for relatively low-dimensional heterogeneous (and homogeneous) time series, methods like \textit{Gaussian processes} and \textit{deep neural networks} can be used, for very high dimensional homogeneous time series we often cannot predict all of the time series together, and for reasons of efficiency and resource distribution we want to predict each component separately and then `consolidate' the predictions.
One approach to that is \textit{hierarchical time series forecasting}.
Some examples of hierarchical time series and their unique characteristics are:

\paragraph*{Demand Forecasting}
\label{par:intoduction:demand_forecasting}
Trying to predict customer demand in order to optimize supply chains.
The forecasts help answering business questions such as: production planning, inventory management, entering a new market.

Characteristics: series comprised of non-negative integers (number of units for each product); forecast horizon might be both for short and long term; time series might have seasonality; series might be irregularly sampled, i.e., not have historical data for the entire period, for example, in a grocery store not all fruits are being sold throughout the year.

The hierarchy of these forecasts can be defined by different possible groupings: (a) the different regions the products are sold; (b) the different customer segments the products are sold to; (c) the different categories each product is associated with.

\paragraph*{Stock Prices}
\label{par:intoduction:stock_prices}
Trying to predict stock prices in order to decide on a trading strategy that will optimize returns.

Characteristics: series are continuous (stock price); forecast horizon might be both for short and long term; time series usually have seasonality and trend; prices can be affected by psychology, i.e. forecasts themselves might change investors' behavior.

The hierarchy of these forecasts can be defined by the different indices each stock is associated with (S\&P500, FTSE100, etc.). A stock index (see~\cite{wiki:stock-index}) is a financial instrument that measures a stock market by grouping multiple stock prices into one number. The aggregation of the different prices into one number may be different between indices, e.g., it can be a simple average or it can be weighted average on the different market cap (total value) of each stock.

\paragraph*{Electricity Demand}
\label{par:intoduction:electricity_demand}
In order to be efficient in managing a power grid, it is vital to be able to forecast electricity demand. The increasing adoption of smart meters provides more data points at the customer level that can increase forecasts accuracy.

Characteristics: series are continuous; forecast horizon might be both for short and long term; time series usually have daily, weekly, and yearly seasonality.

The hierarchy of these forecasts can be defined by geographical regions and customer segments (private, business, etc.). The hierarchy usually has many, could be even millions, of series at the bottom level which might affect computational complexity.

\section*{}

When trying to forecast a hierarchy of many time series simultaneously, and to update forecasts in real time with every new data point that arrives, an issue of computational complexity arises.
In this case, a single forecasting model on a single machine may take too long to train and produce forecasts, and one might choose to use a technique such as \textit{Federated Learning}~\cite{federated-learning} which takes a decentralized approach, i.e., distribute the training to multiple machines, each train on different series independently and then pass the trained models to the central machine to produce forecasts.
Producing forecasts for each series independently is of course not the best solution since the information, imposed by the hierarchy and the relation between the series, is not taken into account.

We can see that, for example, in the \textit{electricity demand} use case, where we have both spatial and semantic hierarchical relation between time series (users in the grid), sharing information between time series should improve individual forecasts. Training that many series (might be millions) in a single-model-single-machine is not tractable in relevant time, and a distributed approach is needed. In the case of distributed approach, an algorithm for sharing information between the forecasts would be beneficial.

The problem we will tackle in this thesis is the following: we are given independent forecasts for each time series in a hierarchy that were (probably) produced by distributed machines; we need to develop a method to improve those forecasts (in a post-processing step) by sharing information between the series using knowledge imposed by the hierarchy.

This problem is an active field of research and, very often, this post-processing method of improving forecasts is referred to as \textit{reconciliation}.
In many works (such as \cite{recon-lin3-MinT}) the goal of this \textit{reconciliation} is to assure coherency of the forecasts, i.e., the forecasts should obey the hierarchical structure. For example, in the \textit{demand forecasting} use case, the forecasts for all stores in a region should add up to the forecast of the aggregated time series of the region.
As stated above, the problem we are trying to solve is related but different. We are trying to improve each of the individual forecasts as opposed to assuring coherency.

\section*{}

The next Chapter~\ref{ch:background} reviews the relevant background needed for the scope of this work. Chapter~\ref{ch:problem_statement} formally defines the problem of hierarchical time series forecasting. Chapter~\ref{ch:theory} presents the new theory we developed to deal with the problem. Chapter~\ref{ch:empirical_evaluation} presents empirical results from experiments we conducted.

\chapter{Background}
\label{ch:background}

\section{Time Series Forecasting}
\label{sec:background:time_series_forecasting}

Time series is any data that is indexed by time, it may be univariate or multivariate. The main tasks of time series analysis are \textit{smoothing}, \textit{filtering} and \textit{forecasting}. \textit{Smoothing} is the task of approximating a state in time, based on all other past and future states. \textit{Filtering} is the task of ``predicting the present'', as in estimating the current not-noisy state, based on all previous and current state. \textit{Forecasting} is predicting the future, based on the past.

Forecasting can be divided into two types: point forecasting and probabilistic or uncertainty forecasting. \textit{Point forecasts} result in a single value that might be univariate or multivariate, while \textit{probabilistic forecasts} result in a probability distribution that quantifies uncertainty.

An overview of methods for time series forecasting is provided below.

\subsection{ARIMA}
\label{subsec:background:forecasting:arima}

Autoregressive integrated moving average (ARIMA) is a method for analyzing time series.
It consists of three different models: autoregressive (AR), differencing (integration) (I), and moving average (MA).
ARIMA model is parameterized by the order term of each of its ``sub'' models, i.e., ARIMA($p, d, q$) is the equivalent of combining: AR($p$), I($d$) and MA($q$).
\cite{box-jenkins} introduced a method for modeling time series with ARIMA models.
Details of these models appear below:

\subsubsection{AR}
    The autoregressive model of order $p$ (AR($p$)) is defined as
    \begin{equation*}
        y_t = \mu + \sum_{i=1}^p \phi_i y_{t-i} + \epsilon_t
    \end{equation*}
    where $\mu$ is a constant term, $\{\phi_i\}_1^p$ are the autoregressive parameters and $\epsilon_t$ is a white noise term for state $t$.
\subsubsection{I}
    Differencing is a transformation applied to a non-stationary time series in order to make it stationary in the mean sense, i.e.: \begin{equation*}
        mean_X(t) = mean_X(t+\tau) \quad \forall \tau \in \mathbb{R}
    \end{equation*}
    For example, an I($1$) differencing is
    \begin{equation*} y'_t = y_t -y_{t-1} \end{equation*}
    and an I($2$) differencing is
    \begin{equation*} y'_t = y_t -2y_{t-1} +y_{t-2} \end{equation*}.
\subsubsection{MA}
    The moving average model of order $q$ (MA($q$)) is defined as
    \begin{equation*} y_t = \mu + \sum_{i=1}^q \theta_i \epsilon_{t-i} + \epsilon_t \end{equation*}
    where $\mu$ is the mean of the series, $\{\theta_i\}_1^q$ are the moving average parameters and $\epsilon_t$ is a white noise term for state $t$.
\subsubsection{Full Model}
    The full ARIMA($p,d,q$) is given by combining all models together, so, for example, an ARIMA($1,1,2$) is given by:
    \begin{equation*} y_t = \mu + y_{t-1} + \phi_1 (y_{t-1} - y_{t-2}) + \theta_1 \epsilon_{t-1} + \theta_2 \epsilon_{t-2} + \epsilon_t \end{equation*}
    Parameters $\{\phi_i\}_1^p, \{\theta_i\}_1^q, \mu$ are estimated by maximum likelihood estimation (MLE). Parameters $p,d,q$ are usually optimized by comparing an information criterion such as Akaike information criterion (AIC) or Bayesian information criterion (BIC).
\subsubsection*{}
ARIMA models try to deal with non-stationary time series by the differencing step, but some trends cannot be removed with differencing. Another approach to deal with non-stationary time series is to use state-space models (Section \ref{subsec:background:forecasting:state-space}).

\subsection{State-Space Models}
\label{subsec:background:forecasting:state-space}

State-space models generally attempt to describe a system by assuming that there is a latent state with time-varying dynamic relationship. The characteristics of the system is that:
\begin{itemize}
    \item It has a Markov property, i.e., the state of the system in time $t$ is enough to infer the state in future time $t+1$.
    \item We cannot observe the latent state itself other than through a noisy observation.
\end{itemize}
This leads to two main equations describing a state-space model --- the \textit{transition equation} which describes how the latent state $\alpha_t$ evolves through time:
\begin{equation*}
    \alpha_{t+1} = f_t( \alpha_t)
\end{equation*}
and the \textit{observation equation} which relates between observed data $y_t$ and latent state:
\begin{equation*}
    y_t = g_t (\alpha_t) + \epsilon_t \quad\quad \epsilon_t \sim \mathcal{D}_t
\end{equation*}

As with ARIMA (\ref{subsec:background:forecasting:arima}) models, parameters of state space models can be estimated with MLE, but can also be estimated with Bayesian inference (see Section~\ref{subsec:background:bayesian_graphical_models:bayesian_ml}).

\subsection{Gaussian Process Regression}
\label{subsec:background:forecasting:gpr}

Gaussian processes (GPs) are models suitable for both regression and classification tasks. Since time series forecasting is a regression task, i.e., we are predicting values based on time-stamped input, GPs are a great tool for time series forecasting. Since this is the subject of this thesis, we continue the background and intuition on regression problems.  One can get more information on GPs for classification (and for regression too) in~\cite{gp-for-ml}.

In regression problems we are given a set of (possibly noisy) observations:
\begin{equation*}
    \mathcal{D} = \{\pmb{x}_t \in \mathbb{R}^D, y_t\}_{t=1}^T
\end{equation*}
or in matrix form:
\begin{equation*}
    X \in \mathbb{R}^{T \times D}, \pmb{y} \in \mathbb{R}^T
\end{equation*}
where we assume the observation model is:
\begin{equation*}
    y_t = f(\pmb{x}_t) + \epsilon_t, \quad \epsilon_t \sim \mathcal{N}(0, \sigma^2)
\end{equation*}
$\epsilon_t$ is random noise i.i.d. normally distributed with variance $\sigma^2$.
We wish to infer the function at new positions $X_*$ which we denote as $\pmb{f}_* \in \mathbb{R}^T$, or in symbols:
\begin{equation*}
    \pmb{f}_* \mid \pmb{y};X,X_*
\end{equation*}

Remarks:
\begin{itemize}
    \item $y_t$ is scalar but of course $\pmb{y}_t$ can be multivariate, in which case we should use what is called \textit{Multi-task Gaussian process} (MTGP). For more information about MTGP see~\cite{bonilla-multi-task-gps}.
    \item $\epsilon$ is Gaussian distributed according to the assumption of a GP. It is possible to have a different noise model, not examined in this work.
    \item we define $\epsilon_t$ as homoscedastic noise but it is possible to define a model with heteroscedastic noise, i.e., $\epsilon_t$ is parameterized by $\sigma_t$ that can change over time. This is not crucial for the understanding of this work and will make derivations more complicated so we stick with homoscedastic noise.
\end{itemize}

Returning to GPs, a GP is a collection of random variables, any finite number of which have a joint Gaussian distribution.
A GP is fully defined by its mean function $m(\pmb{x})$ and covariance function (kernel function) $K(\pmb{x}, \pmb{x}')$:
\begin{align*}
    m(\pmb{x}) &= \mathbb{E}[f(\pmb{x})]\\
    K(\pmb{x}, \pmb{x}') &= \mathbb{E}[(f(\pmb{x}) - m(\pmb{x}))(f(\pmb{x}') - m(\pmb{x}'))]
\end{align*}
and is written:
\begin{equation*}
    f(\pmb{x}) \sim \mathcal{GP}(m(\pmb{x}), K(\pmb{x}, \pmb{x}'))
\end{equation*}
i.e., the function $f$ is GP distributed with mean function $m$ and covariance function $K$.
With this explanation a GP can also be thought of as a distribution over functions.
Since we are dealing with normal distributions it must hold that the matrix $K$ which elements are given by: $K_{i,j}=K(\pmb{x}_i, \pmb{x}_j)$ is a symmetric and positive semi-definite matrix, and thus, the kernel $K(\pmb{x}, \pmb{x}')$ must be a symmetric and positive semi-definite kernel.
 
Returning to the regression problem described here, the full GP model relevant to the problem is defined as:
\begin{equation*}
    \begin{bmatrix}
        \pmb{y}\\
        \pmb{f}_*
    \end{bmatrix} \sim \mathcal{N}\left(\begin{bmatrix} m(X) \\ m(X_*) \end{bmatrix}, \begin{bmatrix} K(X, X) + \sigma^2I & K(X, X_*) \\ K(X_*, X) & K(X_*, X_*) \end{bmatrix}\right)
\end{equation*}
i.e. $\pmb{y}$ and $\pmb{f}_*$ are jointly Gaussian distributed. This is of course after we defined functions $m$ and $K$, which will be examined in Section~\ref{subsubsec:background:forecasting:gpr:kernels}. We can then use the rules of conditioning in a multivariate normal distribution to get:
\begin{align*}
    \pmb{f}_* \mid \pmb{y};X,X_* \sim \mathcal{N}(&m(X_*) + K(X_*, X) [K(X, X) + \sigma^2I]^{-1} (\pmb{y} - m(X)),\\
    &K(X_*, X_*) - K(X_*, X) [K(X, X) + \sigma^2I]^{-1} K(X, X_*))
\end{align*}
Inference in GPs is very simple and requires only linear algebra operations. The downside is that the operation involves inverting a $T \times T$ matrix, where $T$ is the number of observations, which results in $\mathcal{O}(T^3)$ complexity.

The kernel determines almost all the generalization properties of a GP model, and that is why many times the mean function is chosen to be a constant, i.e. $m(\pmb{x})=\pmb{c}$.

We now expand on different kernel functions and how to choose them.

\subsubsection{Kernels}
\label{subsubsec:background:forecasting:gpr:kernels}
A kernel function that is used as a covariance function can be chosen upon the specific problem, but it must be positive semi-definite and symmetric to be valid. I will elaborate on some popular kernels, how to combine kernels in general and on their parameters optimization, some of it is taken from \cite{duvenaud-thesis-2014}.

\paragraph{Standard Kernels}
Some popular kernels and possible derivations are:

\subparagraph{Radial Basis Function}
Radial basis function (RBF) is the default or ``go-to'' kernel in most implementations, and can be defined as
\begin{equation*}
\label{eqn:background:gpr:rbf_kernel}
    K_{RBF}(x,x') = \sigma^2 \exp\left({-\frac{(x-x')^2}{2\ell^2}}\right)
\end{equation*}
where $\ell$, the length scale, controls the length of the ``wiggles'' of the function, and $\sigma^2$ is a scale factor of the distance from the mean.

\subparagraph{Periodic}
A periodic kernel is good for functions that repeat themselves exactly. A periodic kernel can be defined as:
\begin{equation*}
\label{eqn:background:gpr:periodic_kernel}
    K_{Per}(x,x') = \sigma^2 \exp\left({-\frac{2 \sin{(\pi |x-x'| p)}^2}{\ell^2}}\right)
\end{equation*}
where $p$ determines the periodicity and $\ell, \sigma^2$ are the same as in RBF kernels.

\subparagraph{Linear}
The linear kernel is non stationary as opposed to RBF and periodic kernels, i.e., its values are not dependant of the distance between the inputs but on the inputs themselves. A linear kernel can be defined as:
\begin{equation*}
    K_{Lin}(x,x') = \sigma^2 x x'
\end{equation*}
here, again, $\sigma$ is the output scale.

\paragraph{Combining Kernels}
Since positive definiteness is closed under addition and multiplication, to create new kernels, one can combine kernels through addition or multiplication. Generally speaking, one can think of addition of kernels as an `or' operator and multiplying kernels as an `and' operator. For example to create a periodic function with increasing variance in time, one can multiply linear and periodic kernels to create a new kernel, i.e., $K(x, x') = K_{Per}(x,x') K_{Lin}(x,x')$ as defined above.

\paragraph{Optimization of Hyperparameters}
Choosing the right hyperparameters, like $\ell$ and $\sigma$ for the RBF kernel, may be difficult, and in reality practitioners will either optimize those hyperparameters over the set of observations $\mathcal{D}$ with, e.g., maximum likelihood estimation (MLE), or conduct Bayesian inference.

For example MLE in our regression problem would be:
\begin{equation*}
    \pmb{\theta}^{MLE} = \underset{\pmb{\theta}}{\operatorname{\argmax}} \;\mathcal{N}(\pmb{y}|X,\pmb{\theta})
\end{equation*}
where $\pmb{\theta}$ are the hyperparameters of the kernel function chosen, e.g., if the kernel is an RBF kernel and $m(X)=0$, then $\pmb{\theta} = \{\sigma, \ell \}$.

\section{Bayesian Graphical Models}
\label{sec:background:bayesian_graphical_models}

\subsection{Bayesian Machine Learning}
\label{subsec:background:bayesian_graphical_models:bayesian_ml}
Let us recall the Bayes' rule:
\begin{equation}
    P(A|B) = \frac{P(B|A) P(A)}{P(B)}
    \label{eqn:bayes-raw}
\end{equation}
In Bayesian modeling we usually use the following notations:
\begin{itemize}
    \item $y$ - the data or observations.
    \item $\theta$ - the latent (unobserved) parameters.
    \item $\Tilde{y}$ - the predictions of the model.
\end{itemize}
So, when replacing $(A, B)$ with $(\theta, y)$ in~\eqref{eqn:bayes-raw} respectively, we get:
\begin{equation}
    P(\theta|y) = \frac{P(y|\theta) P(\theta)}{P(y)}
    \label{eqn:bayes}
\end{equation}
These probabilities are referred to as:
\begin{equation*}
    \text{posterior} = \frac{\text{likelihood} \times \text{prior}}{\text{evidence}}
\end{equation*}
Note also that the definition of conditional probabilities allows us to write \eqref{eqn:bayes} as:
\begin{equation}
    P(\theta|y) = \frac{P(y|\theta) P(\theta)}{P(y)} = \frac{P(\theta, y)}{P(y)} \propto P(\theta, y)
    \label{eqn:bayes-extended}
\end{equation}
note that $P(y)$ is constant with respect to $\theta$ and only affects the \textit{posterior} as a normalizing constant.

Bayesian modeling for machine learning is done in three main steps:
\begin{enumerate}
     \item \label{bayes-ml-step-1} \textbf{Specifying the generative model}, which is the joint probability of all observable and unobservable quantities in a problem, i.e., $p(\theta, y)$. Can be thought of as the story of how the data came to be.
    \item \label{bayes-ml-step-2} \textbf{Conditioning on observed data} for two main missions: 1) to infer the parameters or getting the posterior distribution, i.e., $p(\theta | y)$, also called Bayesian updating; 2) to get the predictive distribution $p(\Tilde{y} | y)$.
    \item \label{bayes-ml-step-3} \textbf{Evaluating the fit of the model}: How well does the model fit the data? Are the conclusions reasonable? How sensitive are the results to the modeling assumptions?
\end{enumerate}
One can refer to~\cite{bda-gelman, mcelreath2020statistical} to read more on Bayesian data analysis and machine learning.

\subsection{Probabilistic Graphical Models}
\label{subsec:background:bayesian_graphical_models:probabilistic_graphical_models}

A probabilistic graphical model (PGM) is a model that defines a set of random variables in a graph and their joint probability distribution.
There are two types of PGMs: 1) Directed Acyclic Graphs (DAGs) or Bayesian Networks (BNs); 2) Markov Random Fields (MRFs) which are undirected graphs.
We will expand on BNs since they are what we will use in this thesis. More on MRFs (and BNs) can be found in~\cite{koller2009probabilistic}.

\subsubsection{Bayesian Networks}
A \textit{Bayesian network} is a DAG whose nodes represent some random variables, e.g., $(X_1, ..., X_N)$.
The random variables represented in the BN can either be discrete or continuous.
For each node $X_i$ we define the conditional probability distribution as $P(X_i|Pa(X_i))$, where $Pa(X_i)$ denotes the direct parents of $X_i$ (and $P(X_i)$ if $X_i$ has no parents). With that we can express the joint distribution of the variables by the chain rule of BNs:
\begin{equation*}
    P(X_1,...,X_N) = \prod_{i=1}^N P(X_i|Pa(X_i))
\end{equation*}
i.e., the assumption is that each $X_i$ is independent of all other nondescendents given its parents.

In the discrete case we would represent the conditional probability distribution of a specific node in the form of a conditional probability table. For example, if we have two Bernoulli random variables $(X_1, X_2)$, with the following graph:
\begin{center}
\begin{tikzpicture}[node distance={2cm}, thick, main/.style = {draw, circle}] 
    \node[main] (1) {$X_1$};
    \node[main] (2) [right of=1] {$X_2$};
    \draw[->] (1) -- (2);
\end{tikzpicture}
\end{center}
we would need to build these conditional probability tables:
\begin{center}
\begin{tabular}{|c|c|}
    \hline
    $X_1=0$ & $X_1=1$ \\
    \hline\hline
    $p$ & $1-p$ \\
    \hline
\end{tabular}
\quad\quad
\begin{tabular}{|c||c|c|}
    \hline
    & $X_2=0$ & $X_2=1$ \\
    \hline\hline
    $X_1=0$ & $q_0$ & $1-q_0$\\
    \hline
    $X_1=1$ & $q_1$ & $1-q_1$\\
    \hline
\end{tabular}
\end{center}
Where $(p, q_0, q_1)$ are some probabilities.

In the continuous case we would represent the conditional probability distribution of a specific node in the form of a conditional probability density function. For example, if we have two continuous random variables $(\theta, y)$, with the following graph:
\begin{center}
\begin{tikzpicture}[node distance={2cm}, thick, main/.style = {draw, circle}] 
    \node[main] (1) {$\theta$};
    \node[main] (2) [right of=1] {$y$};
    \draw[->] (1) -- (2);
\end{tikzpicture}
\end{center}
we would need to define $p(\theta)$ and $p(y|\theta)$, and this defines the generative model of the problem, i.e.:
\begin{equation*}
    p(\theta, y) = p(\theta) p(y|\theta)
\end{equation*}

\subsection{Soft Evidence}
\label{subsec:background:bayesian_statistics:soft_evidence}
Evidence is a term used for the occurrence of an event, e.g., $X=x$, i.e., the event where random variable $X$ takes the value $x$. This is called \textit{hard evidence}.
Sometimes we are not completely certain that the event has occurred, but we believe that the event has happened with some probability, i.e., we have moved from the prior belief of $p(X)$ to a new belief $p'(X)$. This is called \textit{soft evidence} and is a form of uncertain evidence.

The idea of how to incorporate this new knowledge is discussed in~\cite{jeffrey1990logic} and explained in relation to graphical models in~\cite{pearl-black-book}.

We will use the example in~\cite{jeffrey1990logic, pearl-black-book} to explain this: consider an agent that inspects the color $C$ of a piece of cloth by candle light. That cloth has three possible colors we will denote as $C_1, C_2, C_3$. Before the inspection the agent had prior probability $p(C)$ of the discrete distribution of the color, and after the inspection, he changed his beliefs to $p'(C)$. Of course he could not be sure of the color he saw since it was seen with the light of the candle.

Notice that a usual relation in a graphical model (Section \ref{subsec:background:bayesian_graphical_models:probabilistic_graphical_models}) will be that the color $C$ is a latent parameter that is being observed in the evidence $e$, but with \textit{soft evidence} we have the other way around where the evidence $e$ changed our belief in the latent parameter $C$, i.e., we have distribution $C | e$, rather than a $e | C$, in our generative model.

Back to our example, let us imagine that we have a new parameter $S$ that denotes the chance that this cloth will be sold and that $S$ depends only on the cloth's color, i.e., $S | C$.
A picture of this network is presented in Figure~\ref{fig:background:bayesian_statistics:soft_evidence:example}.
We would like to update our belief in $p(S)$ to $p'(S)$, i.e., the new belief after the \textit{soft evidence} $e$ --- that is done with \textit{Jeffrey's Rule} of updating:
\begin{equation*}
    p'(S) = \sum_i p(S | C_i) p'(C_i)
\end{equation*}
In that way we can incorporate the knowledge acquired by the evidence on our parameter $S$ which is not directly connected to the evidence.

\begin{figure}[h!]	
	\centering
	{\begin{tikzpicture}
    \node[latent] (C) {$C$};
    \node[latent,left=of C,yshift=-0.25cm] (e) {$e$};
    \node[latent,right=of C,yshift=-0.25cm] (S) {$S$};
    \edge {C} {S};
    \edge [-,decorate,decoration=snake] {C} {e};
\end{tikzpicture}}
	\caption{A network representing the graphical model described in the example in~\ref{subsec:background:bayesian_statistics:soft_evidence}. Undirected snake edge between $e$ and $C$ denotes soft evidence about $C$ upon observing $e$.}
	\label{fig:background:bayesian_statistics:soft_evidence:example}
\end{figure}
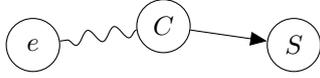

Another way to incorporate uncertain evidence in a graphical model is with \textit{virtual evidence}.
The basic idea is to add a virtual node $V$ in the graph that always takes a certain value (with probability 1), and is connected to the node of the parameter it convey evidence to (in our example it is node $C$). Then the likelihood $p(V=1 | e)$ needs to be defined and added that to the joint distribution of the graph.
This is not the method we use in this work, so we will not add more on this subject, but one can refer to~\cite{pearl-black-book} for more information.

\section{Other Related Work}
\label{sec:background:other_related_work}

\subsection{Forecast Reconciliation}
\label{sec:related_work:forecast_reconciliation}
\textit{Reconciliation} is an active field of research, and in other related works it is done in order to ensure \textit{coherency} of the different forecasts of the hierarchy, i.e. forecast for each level of the hierarchy should be equal to the aggregation of its descendants.

As stated in Chapter~\ref{ch:intro}, we would take a different approach of reconciling the individual forecasts by sharing information between the series in order to improve each individual forecast, rather then assuring coherency. This approach is well described in Chapter~\ref{ch:problem_statement}.
Nonetheless, we will note main works in this field.
Works can be divided to two main approaches: \textit{Linear Reconciliation} and \textit{Probabilistic Reconciliation}.

\subsubsection{Linear Reconciliation}
\cite{recon-lin1} introduced a least squares method, that was later improved by MinT \cite{recon-lin3-MinT}, which assumes the base forecasts are unbiased, while trying to produce the minimum variance unbiased revised forecasts. A different approach, ERM \cite{recon-lin4-unbiased}, does not assume unbiased forecasts, and tries to minimize the mean squared revised forecast errors. Other linear reconciliation works include \cite{recon-lin2, recon-lin5}.
    
\subsubsection{Probabilistic Reconciliation}
\cite{recon-prob1-copulas} propose using copulas and solving a LASSO problem to get a hierarchical coherent probabilistic forecasts.
\cite{recon-prob2-bayesian} propose a two level Bayesian approach, firstly learning the distribution of bottom level forecasts through estimation of the covariance matrix by historical errors, and then obtain aggregate consistent point forecasts, by defining a problem-specific loss function.
\cite{hier-bayes-paper} propose a fully Bayesian approach of reconciling the forecasts in order to make them coherent, and we will evaluate our method against theirs in Chapter~\ref{ch:empirical_evaluation}.

\chapter{Problem Statement}
\label{ch:problem_statement}

Let us define forecast reconciliation formally. In this work, we limit the discussion to reconciliation of scalar real-valued forecasts. Reconciliation of multiple multidimensional forecasts follows the same lines but involves a much more elaborate theoretical analysis, with limited contribution to understanding of the principles of Bayesian reconciliation.

\begin{definition}[Forecast Reconciliation]
\label{dfn:reconciliation}
Consider a system consisting of $N$ subsystems.  Assume that the parameter of interest of the $i$-th subsystem is $x_i \in \mathbb{R}$, so that the vector of parameters of interest of the whole system is $\pmb{x} \in \mathbb{R}^N$.  Forecasts at a new location or a certain time in the future are obtained 
    \begin{itemize}
        \item for each $x_i$
        \item for $u=f(\pmb{x})$ for some $f: \mathbb{R}^N \to \mathbb{R}^M$, $M \le N$
    \end{itemize}
in the form of distributions $\Theta_i$ with density $\theta_i$ for each $x_i$ and $\Eta$ with density $\eta$ for $u$. Then, \textbf{forecast reconciliation} consists of inferring $\pmb{x}|\Theta, \Eta$.
\end{definition}

Function $f$ can be interpreted as computing summary statistics of some form on $\pmb{x}$.
A popular form of $f$ is $A\pmb{x}$ where $A$ is a matrix of shape $M \times N$ with all elements equal to either 1 or 0, such that each element of $u$ is the sum of a subset of components of $\pmb{x}$. In this case, $A$ apparently reflects the grouping of subsystems such that all subsystems corresponding to entries of 1 in a row of $A$ belong to the same group. Groups may overlap.
However, other forms of $f$, both linear and general, are conceivable; for example, $f$ may compute the mean, the product, or the sum of squares of $\pmb{x}$.
It is worth noting that $x_i$ are forecast independently, but $u$ must be forecast as a whole; and that is why we set $\pmb{x}$, but not $u$, in bold. Intuitively, $u$ is a small set of summary statistics of $\pmb{x}$ that conveniently admits simultaneous forecasting.

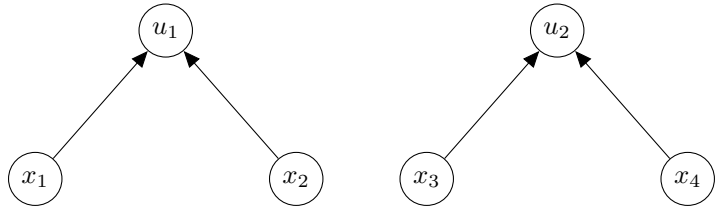
\begin{figure}[h!]	
	\centering
	{\begin{tikzpicture}
    \node[latent] (x_1) {$x_{1}$};
    \node[latent,right=of x_1,color=white] (x_15) {};
    \node[latent,right=of x_15] (x_2) {$x_{2}$};
    \node[latent,right=of x_2] (x_3) {$x_{3}$};
    \node[latent,right=of x_3,color=white] (x_35) {};
    \node[latent,right=of x_35] (x_4) {$x_{4}$};
    \node[latent,above=of x_15,yshift=0.25cm] (u_1) {$u_1$};
    \node[latent,above=of x_35,yshift=0.25cm] (u_2) {$u_2$};
    \edge {x_1} {u_1};
    \edge {x_2} {u_1};
    \edge {x_3} {u_2};
    \edge {x_4} {u_2};
\end{tikzpicture}}
	\caption{Example of a system with four subsystems belong to two groups.}
	\label{fig:problem_statement:example}
\end{figure}
Consider the following example of the system appearing in Figure~\ref{fig:problem_statement:example}. In this system we have $N=4$ subsystems in $M=2$ groups.
In this example $f$ can be any function such that $f: \mathbb{R}^4 \to \mathbb{R}^2$. Continuing the example, we define $f(\pmb{x})=A\pmb{x}$, and if the summary statistics represent a summation, then:
\begin{equation*}
    \pmb{x} = \begin{bmatrix}
        x_1 & x_2 & x_3 & x_4
    \end{bmatrix}^T
\end{equation*}
and 
\begin{equation*}
    A = \begin{bmatrix}
            1 & 1 & 0 & 0 \\
            0 & 0 & 1 & 1
    \end{bmatrix}
\end{equation*}
As stated in Definition~\ref{dfn:reconciliation}, we are also given forecasts for each subsystem and for $u$, and in our example  those forecasts could be normally distributed:
\begin{itemize}
    \item $\pmb{x}_i \sim \Theta_i \equiv \mathcal{N}(\mu_i, \sigma_i^2) \quad \forall i \in \{1, 2, 3, 4\}$
    \item $u \sim \Eta \equiv \mathcal{MVN}(\pmb{\mu}_u, \Sigma_u)$
\end{itemize}
and we would like to infer $\pmb{x}|\Theta, \Eta$.

Notice that although $M=2$, we have a single (multivariate normally distributed) forecast for $u$.
Connecting this example to the introduction in Chapter~\ref{ch:intro}, we are assuming that although it is intractable to forecast $\pmb{x}$ in a single model when $N$ is very large (of course in this toy example it is not the case), it is tractable to forecast $M$ when $M<<N$ in a single model, hence we have a single forecast distribution for $u$.

\chapter{Theoretical Analysis}
\label{ch:theory}
\section{Generative Model}
\label{sec:theory:model}
\begin{figure}	
	\centering
	\resizebox{0.5\linewidth}{!}{\begin{tikzpicture}
    \node [latent] (s) {$s$};
    \node [obs,left=of s,yshift=-0.5cm] (y) {$y$};
    \node [latent,right=of s,yshift=-0.5cm] (x) {$x$};
    \edge {s} {x};
    \edge {s} {y};
\end{tikzpicture}}
	\caption{A forecasting model. The latent state $s$ is observed through observations $y$. A forecasted parameter of interest $x$ depends on $s$.}
	\label{fig:forecast-model}
\end{figure}
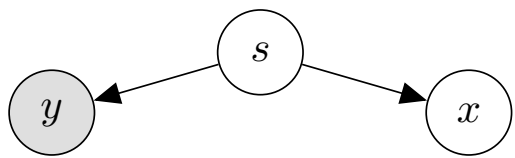

We approach the problem of forecast reconciliation using the methodology of Bayesian generative modelling (Section~\ref{sec:background:bayesian_graphical_models}). First, we describe the full model. Then, we reduce the full network to a subnetwork relevant for reconciliation of forecasts of $x_i$.
Before turning to forecast
reconciliation, let us describe in general terms the structure
of  a Bayesian network for forecasting.
\begin{definition}
A Bayesian forecasting model is a generative model (Figure~\ref{fig:forecast-model}):
\begin{equation}
	\begin{aligned}
		s & \sim D_s \\
		y|s & \sim D_y(s) \\
		x|s & \sim D_x(s)
	\end{aligned}
	\label{eqn:forecast-model}
\end{equation}
where 
\begin{itemize}
\item $s$ is a latent variable representing the state of
the system;
\item $y$ is the observation through which $s$ is
observed, i.e., historical values;
\item $x$ is the parameter of interest to be forecasted, i.e., future states;
\item $D_s$, $D_y(s)$, and $D_x(s)$ are distributions.
\end{itemize}
A forecast is the marginal distribution of $x$ given $y$.
\end{definition}

There are $N+1$ forecasting models involved in the reconciliation problem (Definition~\ref{dfn:reconciliation}): \begin{itemize}
\item one model for forecasting $u$;
\item $N$ models for forecasting each of $x_i$, $i \in 1 ... N$.
\end{itemize}
We assume that all forecasting models have access to the same observation $y$. Instead of the prior and conditional
distributions for each of the models, we are given the forecasts --- the posterior marginal distributions $\Theta_i$ of $x_i$ and
$\Eta$ of $u$.

Based on the forecast of $u$, we aim to update
the forecasts of $x_i$. To formalize this, we imagine yet
another forecasting model, for the whole system, for which
neither the prior or the conditional distributions, nor the
marginal posterior of $\pmb{x}$ given $y$ are given, but which is probabilistically related to the forecasting models for $x_i$ and $u$. The full model, combining the individual forecast models for $x_i$ and $u$ as well as the imagined forecasting model for $\pmb{x}$, is shown in Figure~\ref{fig:full-model}.

\begin{figure}
	\begin{subfigure}{0.45\linewidth}
		\centering
		\resizebox{0.9\linewidth}{!}{\begin{tikzpicture}
\node [latent,line width=1.44] (x) {$x$};
\node [latent,right=of x,yshift=-1cm,line width=1.44] (u) {$u$};
\node [latent,below=of x,yshift=-2.5cm] (xhat) {$\hat x_i$};
\node [latent,right=of xhat,yshift=+1cm] (uhat) {$\hat u$};
\node [latent,left=of x, xshift=-0.5cm] (s) {$s$};
\node [latent,left=of xhat,xshift=-0.5cm] (sx) {$s_i$};
\node [latent,left=of uhat,xshift=-0.5cm] (su) {$s_u$};
\node [obs,below=of s,line width=1.44] (y) {$y$};
\edge {s} {x};
\edge {sx} {xhat};
\edge {su} {uhat};
\edge {s} {y};
\edge {sx} {y};
\edge {su} {y};
\edge [line width=1.44] {x} {u};
\edge [dashed] {x} {xhat};
\edge [dashed] {u} {uhat};
\plate [] {xplate} {(sx) (xhat)} {$N$};
\end{tikzpicture}}
		\caption{full network}
		\label{fig:full-model}
	\end{subfigure}
	\begin{subfigure}{0.45\linewidth}
		\centering
		\resizebox{0.9\linewidth}{!}{\begin{tikzpicture}
\node [latent] (x) {$x$};
\node [latent,right=of x,yshift=-1cm] (u) {$u$};
\phantom{\node [latent,below=of x,yshift=-2.5cm] (xhat) {$\hat x_i$};}
\phantom{\node [latent,right=of xhat,yshift=+1cm] (uhat) {$\hat u$};}
\phantom{\node [latent,left=of x, xshift=-0.5cm] (s) {$s$};}
\phantom{\node [latent,left=of xhat,xshift=-0.5cm] (sx) {$s_i$};}
\phantom{\node [latent,left=of uhat,xshift=-0.5cm] (su) {$s_u$};}
\node [obs,below=of s] (y) {$y$};
\phantom{\edge {s} {x};}
\phantom{\edge {sx} {xhat};}
\phantom{\edge {su} {uhat};}
\phantom{\edge {s} {y};}
\phantom{\edge {sx} {y};}
\phantom{\edge {su} {y};}
\edge {x} {u};
\edge [-,decorate,decoration=snake] {u} {y};
\phantom{\edge [dashed] {x} {xhat};}
\phantom{\edge [dashed] {u} {uhat};}
\phantom{\plate [] {xplate} {(sx) (xhat)} {$N$};}
\end{tikzpicture}}
		\caption{relevant subnetwork}
		\label{fig:reduced-model}
	\end{subfigure}
	\caption{Bayesian network for forecast reconciliation. The undirected snake edge between $u$ and $y$ denotes soft evidence about $u$ upon observing $y$.}
 	\label{fig:model}
\end{figure}
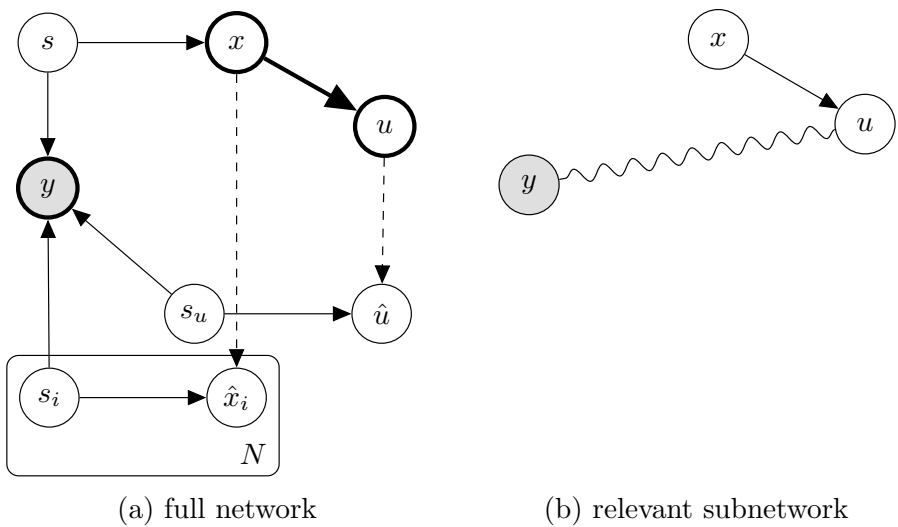

In the full model, 
\begin{itemize}
\item $s$ and $\pmb{x}$ are the latent state and the forecasted parameter of interest of the imagined forecasting model for the whole system;
\item $s_i$ and $\hat x_i$, $i \in 1 ... N$, are the latent state and the forecasted parameter of interest of the $i$-th subsystem;
\item $s_u$ and $\hat u$ are the latent state and the forecasted parameter of the summary $u=f(\pmb{x})$.
\end{itemize}
Observation $y$ is shared by all forecasting models.
Dashed edges denote probabilistic relationships between $\pmb{x}$ and $\hat x_i$ and $u=f(\pmb{x})$ and $\hat u$.
However, Definition~\ref{dfn:reconciliation} does not specify these relationships.
We assume that 
\begin{itemize}
\item the forecast $\hat u$ is accurate, that is $p(u|y) \equiv p(\hat u|y)$;
\item although $\pmb{x}|y$ cannot be constructed from $\hat x_i|y$ because interdependencies between $x_i$ are unknown, the product distribution of $\hat x_i$ is suitable as a prior for $\pmb{x}$.
\end{itemize}
Based on these assumptions, we reduce the full network in Figure~\ref{fig:full-model} to a subnetwork in Figure~\ref{fig:reduced-model} required for reconciliation, dropping parts of the full network with unknown prior or conditional probabilities. In particular, since neither the prior distributions of state variables $s$, $s_i$, $s_u$, nor the conditional distributions of observations or forecasts given the state variables are known, we omit them in the subnetwork in Figure~\ref{fig:reduced-model}.
The snake-shaped undirected link between $y$ and $u$ denotes that $y$ provides \textit{soft evidence}~\ref{subsec:background:bayesian_statistics:soft_evidence} about $u$, that is, the posterior distribution $u|y$, rather than the conditional distribution $y|u$, is given.

\begin{definition} A Bayesian forecast reconciliation model is a Bayesian generative model (Figure~\ref{fig:reduced-model}):
\begin{equation}
	\begin{aligned}
		\pmb{x} & \sim \Theta_1 \times \Theta_2 \times ...  \times \Theta_N \\
		u|\pmb{x} & \sim \delta(f(\pmb{x})) \\
		u|y & \sim \Eta
	\end{aligned}
	\label{eqn:model}
\end{equation}
That is,
\begin{itemize}
\item  $\pmb{x}$ is the latent variable to be inferred, with the prior constructed as the product of marginal forecasts of its components;
\item $u$ is a nuisance latent variable with Dirac delta density $p(u|\pmb{x})=\delta(u-f(\pmb{x}))$;
\item observation $y$ provides soft evidence about $u$.
\end{itemize}
\label{dfn:model}
\end{definition}

The objective of inference on \eqref{eqn:model} is the posterior distribution $\pmb{x}|y$.
The computation of $p(\pmb{x}|y)$ is formalized by the following proposition:
\begin{proposition}
\label{pro:post}
Under the assumptions of model~\eqref{eqn:model},
\begin{equation}
    p(\pmb{x}| y) = p(\pmb{x}) \frac {p(u=f(\pmb{x})|y)} {p(u=f(\pmb{x}))},
	\label{eqn:pro}
\end{equation}
where
\begin{align}
		\label{eqn:pro-x} p(\pmb{x}) & \triangleq \prod_{i=1}^N \theta_i(x_i) \\
		\label{eqn:pro-u-given-y} p(u|y) & \triangleq \eta(u) \\
		\label{eqn:pro-u} p(u) & \triangleq  \int_X \delta(u-f(\pmb{x}))p(\pmb{x})d\pmb{x}
\end{align}
\end{proposition}
\begin{proof}
\begin{equation}
p(\pmb{x}|y) = \int_U p(\pmb{x}|u) p(u|y) du
\label{eqn:pro-proof-1}
\end{equation}
$p(u|y)=\eta(u)$ is known, thus we only need to express $p(\pmb{x}|u)$ in known terms.
By the Bayes rule,
\begin{equation}
p(\pmb{x}|u) = p(\pmb{x}) \frac {p(u|\pmb{x})} {p(u)} = p(\pmb{x}) \frac {\delta(u-f(\pmb{x}))} {p(u)}
\label{eqn:pro-proof-2}
\end{equation}
Substituting~\eqref{eqn:pro-proof-2} into~\eqref{eqn:pro-proof-1}, we obtain
\begin{equation}
\begin{aligned}
p(\pmb{x}|y)  & = \int_U p(\pmb{x})\frac  {\delta(u-f(\pmb{x}))} {p(u)} p(u|y) du  \\
& = p(\pmb{x}) \frac {p(u=f(\pmb{x})|y)} {p(u=f(\pmb{x}))}
\end{aligned}
\end{equation}
\end{proof}

\subsection{Linear Gaussian Reconciliation}
\label{subsec:theory:model:linear-model}
In applications, such as those mentioned in Chapter~\ref{ch:problem_statement}, reconciliation is often performed with $f(\pmb{x}) = A\pmb{x}$ and Gaussian forecast distributions.
In this case, $p(\pmb{x}|y)$ can be obtained in the closed form of the density of a multivariate Gaussian distribution, as formalized by the following proposition:
\begin{proposition}
Let $f(\pmb{x}) = A\pmb{x}$.
Let the forecast distributions $\Theta$ and $\Eta$ be the Gaussian distributions, parameterized by $(\pmb{\mu}_{\theta}, \Sigma_{\theta})$, $(\pmb{\mu}_{\eta}, \Sigma_{\eta})$, i.e., 
\begin{itemize}
    \item $p(\pmb{x}) = p_\mathcal{N}(\pmb{x} | \pmb{\mu}_{\theta}, \Sigma_{\theta})$, 
    \item $p(f(\pmb{x}) | y)=p_\mathcal{N}(f(\pmb{x}) | \pmb{\mu}_{\eta}, \Sigma_{\eta})$.
\end{itemize}
Then,
\begin{equation}
p(\pmb{x}| y) \propto p_\mathcal{N}(\pmb{x} | \pmb{\mu}_{LG}, \Sigma_{LG})
\label{eqn:pro-lg}
\end{equation}
where
\begin{align}
    \label{pro:lg:sigma-lg} \Sigma_{LG} &= (\Sigma_\theta^{-1} + A^T [\Sigma_\eta^{-1} - (A \Sigma_\theta A^T)^{-1}] A)^{-1}\\
    \label{pro:lg:mu-lg} \pmb{\mu}_{LG} &= \pmb{\mu}_\theta + \Sigma_{LG} A^T \Sigma_\eta^{-1} (\pmb{\mu}_\eta - A \pmb{\mu}_\theta)
\end{align}
\label{pro:lg}
\end{proposition}

\begin{proof}
$f(\pmb{x}) = A\pmb{x}$ is an affine transformation on the multivariate Gaussian distributed parameter $\pmb{x}$, which implies:
\begin{equation}
    p(f(\pmb{x}) | \pmb{\theta}) = p_\mathcal{N}(f(\pmb{x}) | A\pmb{\mu}_{\theta}, A\Sigma_{\theta}A^T)
    \label{eqn:pro-lg-proof:affine}
\end{equation}
Using~\eqref{eqn:pro} with~\eqref{eqn:pro-lg-proof:affine} and the distributions in Proposition~\ref{pro:lg}:
\begin{align}
    \label{eqn:pro-lg-proof:1} p(\pmb{x}|y) &= p_\mathcal{N}(\pmb{x} | \pmb{\mu}_{\theta}, \Sigma_{\theta}) \frac{p_\mathcal{N}(f(\pmb{x}) | \pmb{\mu}_{\eta}, \Sigma_{\eta})}{p_\mathcal{N}(f(\pmb{x}) | A\pmb{\mu}_{\theta}, A\Sigma_{\theta}A^T)}\\
    \label{eqn:pro-lg-proof:2} &\propto \exp \left[(\pmb{x} - \pmb{\mu}_\theta)^T \Sigma_\theta^{-1} (\pmb{x} - \pmb{\mu}_\theta)\right] \frac{\exp \left[(A\pmb{x} - \pmb{\mu}_\eta)^T \Sigma_\eta^{-1} (A\pmb{x} - \pmb{\mu}_\eta)\right]}{\exp \left[(A\pmb{x} - A \pmb{\mu}_\theta)^T (A \Sigma_\theta A^T)^{-1} (A\pmb{x} - A \pmb{\mu}_\theta) \right]}\\
    \label{eqn:pro-lg-proof:3-short-exp} &= \exp (g(\pmb{x}))
\end{align}
Here $g(\pmb{x})$ represents the grouping of all the terms in the exponents of~\eqref{eqn:pro-lg-proof:2}.\\
Expanding the terms in the exponent of~\eqref{eqn:pro-lg-proof:2}:
\begin{equation}
    \begin{aligned}
    g(\pmb{x}) = &\pmb{x}^T \Sigma_\theta^{-1} \pmb{x} +\pmb{x}^T A^T \Sigma_\eta^{-1} A \pmb{x} -\pmb{x}^T A^T (A \Sigma_\theta A^T)^{-1} A \pmb{x}\\
    &-2 \pmb{\mu}_\theta^T \Sigma_\theta^{-1} \pmb{x} -2 \pmb{\mu}_\eta^T \Sigma_\eta^{-1} A \pmb{x} +2 \pmb{\mu}_\theta^T A^T (A \Sigma_\theta A^T)^{-1} A \pmb{x}\\
    &+\pmb{\mu}_\theta^T \Sigma_\theta^{-1} \pmb{\mu}_\theta + \pmb{\mu}_\eta^T \Sigma_\eta^{-1} \pmb{\mu}_\eta -\pmb{\mu}_\theta^T A^T (A \Sigma_\theta A^T)^{-1} A \pmb{\mu}_\theta
    \end{aligned}
    \label{eqn:pro-lg-proof:exp1}
\end{equation}
Let us denote the terms in the third line of~\ref{eqn:pro-lg-proof:exp1} by $C$ defined as follows:
\begin{equation*}
    C = \pmb{\mu}_\theta^T \Sigma_\theta^{-1} \pmb{\mu}_\theta + \pmb{\mu}_\eta^T \Sigma_\eta^{-1} \pmb{\mu}_\eta -\pmb{\mu}_\theta^T A^T (A \Sigma_\theta A^T)^{-1} A \pmb{\mu}_\theta
\end{equation*}
Here $C$ does not depend on $\pmb{x}$ and only affects the normalization constant of the density of Equation~\eqref{eqn:pro-lg}. Replacing $C$ in~\ref{eqn:pro-lg-proof:exp1}:
\begin{equation}
    \begin{aligned}
    g(\pmb{x}) = &\pmb{x}^T \Sigma_\theta^{-1} \pmb{x} +\pmb{x}^T A^T \Sigma_\eta^{-1} A \pmb{x} -\pmb{x}^T A^T (A \Sigma_\theta A^T)^{-1} A \pmb{x}\\
    &-2 \pmb{\mu}_\theta^T \Sigma_\theta^{-1} \pmb{x} -2 \pmb{\mu}_\eta^T \Sigma_\eta^{-1} A \pmb{x} +2 \pmb{\mu}_\theta^T A^T (A \Sigma_\theta A^T)^{-1} A \pmb{x}\\
    &+C
    \end{aligned}
    \label{eqn:pro-lg-proof:exp2}
\end{equation}

Grouping the terms in~\eqref{eqn:pro-lg-proof:exp2}:
\begin{equation}
    \begin{aligned}
    g(\pmb{x}) = &\pmb{x}^T (\Sigma_\theta^{-1} + A^T \Sigma_\eta^{-1} A - A^T (A \Sigma_\theta A^T)^{-1} A) \pmb{x}\\
    &-2 (\pmb{\mu}_\theta^T \Sigma_\theta^{-1} + \pmb{\mu}_\eta^T \Sigma_\eta^{-1} A - \pmb{\mu}_\theta^T A^T (A \Sigma_\theta A^T)^{-1} A) \pmb{x}\\
    &+C
    \end{aligned}
    \label{eqn:pro-lg-proof:exp3}
\end{equation}
Adding and subtracting the term $(\pmb{\mu}_\theta^T A^T \Sigma_\eta^{-1} A)$ from the parentheses in the second line of~\eqref{eqn:pro-lg-proof:exp3}:
\begin{equation}
    \begin{aligned}
    g(\pmb{x}) = &\pmb{x}^T (\Sigma_\theta^{-1} + A^T \Sigma_\eta^{-1} A - A^T (A \Sigma_\theta A^T)^{-1} A) \pmb{x}\\
    &-2 (\pmb{\mu}_\theta^T \Sigma_\theta^{-1} + \pmb{\mu}_\theta^T A^T \Sigma_\eta^{-1} A - \pmb{\mu}_\theta^T A^T (A \Sigma_\theta A^T)^{-1} A + \pmb{\mu}_\eta^T \Sigma_\eta^{-1} A - \pmb{\mu}_\theta^T A^T \Sigma_\eta^{-1} A) \pmb{x}\\
    &+C
    \end{aligned}
    \label{eqn:pro-lg-proof:exp4}
\end{equation}
Grouping terms:
\begin{equation}
    \begin{aligned}
    g(\pmb{x}) = &\pmb{x}^T (\Sigma_\theta^{-1} + A^T \Sigma_\eta^{-1} A - A^T (A \Sigma_\theta A^T)^{-1} A) \pmb{x}\\
    &-2 (\pmb{\mu}_\theta^T [\Sigma_\theta^{-1} + A^T \Sigma_\eta^{-1} A - A^T (A \Sigma_\theta A^T)^{-1} A] + \pmb{\mu}_\eta^T \Sigma_\eta^{-1} A - \pmb{\mu}_\theta^T A^T \Sigma_\eta^{-1} A) \pmb{x}\\
    &+C
    \end{aligned}
    \label{eqn:pro-lg-proof:exp5}
\end{equation}
Introducing the term $\Sigma_{LG}$ as defined in Equation~\eqref{pro:lg:sigma-lg} and replacing in~\eqref{eqn:pro-lg-proof:exp5}:
\begin{equation}
    \begin{aligned}
    g(\pmb{x}) = &\pmb{x}^T \Sigma_{LG}^{-1} \pmb{x}\\
    &-2 (\pmb{\mu}_\theta^T \Sigma_{LG}^{-1} + \pmb{\mu}_\eta^T \Sigma_\eta^{-1} A  \Sigma_{LG} \Sigma_{LG}^{-1} - \pmb{\mu}_\theta^T A^T \Sigma_\eta^{-1} A  \Sigma_{LG} \Sigma_{LG}^{-1}) \pmb{x}\\
    &+ C
    \end{aligned}
    \label{eqn:pro-lg-proof:exp6}
\end{equation}
Grouping terms:
\begin{equation}
    \begin{aligned}
    g(\pmb{x}) = &\pmb{x}^T \Sigma_{LG}^{-1} \pmb{x}\\
    &-2 (\pmb{\mu}_\theta + \Sigma_{LG} A^T \Sigma_\eta^{-1} \pmb{\mu}_\eta - \Sigma_{LG} A^T \Sigma_\eta^{-1} A \pmb{\mu}_\theta)^T \Sigma_{LG}^{-1} \pmb{x}\\
    &+ C
    \end{aligned}
    \label{eqn:pro-lg-proof:exp7}
\end{equation}
Introducing the term $\pmb{\mu}_{LG}$ as defined in Equation~\eqref{pro:lg:mu-lg} and substituting it in~\eqref{eqn:pro-lg-proof:exp7}:
\begin{equation}
g(\pmb{x}) = \pmb{x}^T \Sigma_{LG}^{-1} \pmb{x} -2 \pmb{\mu}_{LG}^T \Sigma_{LG}^{-1} \pmb{x} + C
\label{eqn:pro-lg-proof:exp8}
\end{equation}
Plugging Equation~\eqref{eqn:pro-lg-proof:exp8} in Equation~\eqref{eqn:pro-lg-proof:2}:
\begin{align}
    p(\pmb{x}|y) &\propto \exp \Big( \pmb{x}^T \Sigma_{LG}^{-1} \pmb{x} -2 \pmb{\mu}_{LG}^T \Sigma_{LG}^{-1} \pmb{x} + C \Big)\\
    &\propto \exp \Big( (\pmb{x} - \pmb{\mu}_{LG})^T \Sigma_{LG}^{-1} (\pmb{x} - \pmb{\mu}_{LG}) \Big)\\
    &\propto p_\mathcal{N}(\pmb{x} | \pmb{\mu}_{LG}, \Sigma_{LG})
\end{align}
Which completes the proof.
\end{proof}

\chapter{Empirical Evaluation}
\label{ch:empirical_evaluation}

In this chapter we evaluate the \textit{Linear Gaussian} model described in Section~\ref{subsec:theory:model:linear-model}.

In each evaluation (experiment) we compare two reconciliation methods --- the first method is the one presented in this thesis, Section~\ref{subsec:theory:model:linear-model}, which we denote by \textit{ET Reconciler}, the second is the \textit{LG} method from \cite{hier-bayes-paper} which we denote by \textit{Bayes Paper LG Reconciler}. For each reconciliation method tested we produce a reconciled probability $\tilde{\mathcal{D}}_t$ for each time step $t$.

We perform those experiments on two datasets --- synthetic and real data set, where these evaluations are described in Section~\ref{sec:experiments:synthetic} and~\ref{sec:experiments:real} respectively.
Generally, we  fit a GP to each series to obtain the base forecast, then we apply the reconciliation method to try to improve these forecasts, and compute the metric described in Section~\ref{sec:experiments:metrics} before and after the reconciliation.
On the synthetic data set we  performed two  evaluations (or experiments) that differ in the way the base forecasts are generated --- this will be described in Section~\ref{sec:experiments:synthetic}.

\section{Metrics}
\label{sec:experiments:metrics}

For each reconciliation method we compute the \textit{Negative Log Predictive Density} (NLPD):
\begin{equation}
    \text{NLPD} = \frac{1}{T} \sum_t -\log{p_{\tilde{\mathcal{D}}_t}(\pmb{x}_t)}
\end{equation}
i.e., the average, over the time steps, of the negative log of the true values $\pmb{x}$ under the reconciled probability density function $\tilde{\mathcal{D}}$.
For each data set we perform $N=50$ simulations and compute the \textit{NLPD}. We report the mean, the standard error of the mean (denoted as SEM) and the standard deviation (denoted as SD) of the \textit{NLPD}s of these $N$ simulations for each reconciliation method (and for the base forecasts as a benchmark).

\section{Synthetic Data Set}
\label{sec:experiments:synthetic}

\subsection{Data --- Experiment A}
\label{subsec:experiments:synthetic:data-a}

\begin{figure}	
	\centering
	\resizebox{0.5\linewidth}{!}{\begin{tikzpicture}
    \node[latent] (u) {$u$};
    \node[latent,left=of u,yshift=-1cm] (x_2) {$x^{(2)}$};
    \node[latent,left=of x_2] (x_1) {$x^{(1)}$};
    \node[latent,right=of u,yshift=-1cm] (x_3) {$x^{(3)}$};
    \node[latent,right=of x_3] (x_4) {$x^{(4)}$};
    \edge {x_1} {u};
    \edge {x_2} {u};
    \edge {x_3} {u};
    \edge {x_4} {u};
\end{tikzpicture}}
	\caption{The synthetic data set for Experiment~\ref{subsec:experiments:synthetic:data-a}. The four systems $\{x^{(n)}\}_{n=1}^4$ are connected to the same parent $u$.}
	\label{fig:experiments:synthetic:data_a}
\end{figure}
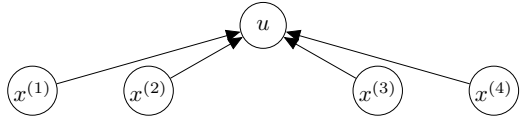
The first experiment includes a synthetic dataset consisting of a hierarchy of 4 series all connected to same parent, as shown in Figure~\ref{fig:experiments:synthetic:data_a}.
Each series is a sine wave, corrupted by noise. The observations are the first period of the sine wave and the task is to predict the second cycle.

Observations for each series are denoted by $y^{(i)}_t$, and the future (hidden) states by $x^{(i)}_t$ for the bottom series and $u_t$ for the sum of the series. A bold symbol, $\pmb{x}^{(n)} \in \mathbb{R}^T$, is a vector of all time steps, and $\pmb{x}_t \in \mathbb{R}^N$, is a vector of all series in time step $t$.

Data was created: 
\begin{align}
    y_t^{(n)} &= \sin (t) + \epsilon_t^{(n)}, \quad t \in [0, 2\pi), \quad \epsilon_t^{(n)} \sim \mathcal{N}(0, \sigma_{\epsilon}) \quad \forall n=\{1,2,3,4\}\\
    x_t^{(n)} &= \sin (t) + \epsilon_t^{(n)}, \quad t \in [2\pi, 4\pi], \quad \epsilon_t^{(n)} \sim \mathcal{N}(0, \sigma_{\epsilon}) \quad \forall n=\{1,2,3,4\}
\end{align}
In the different experiments on this data set we experimented with the following noise levels: $\sigma_{\epsilon} \in \{0, 0.1, 0.2, 0.5, 1, 2, 5\}$.

The sum of series is given by:
\begin{equation}
    u_t = f(\pmb{x}_t) = \sum_{i=1}^4 x^{(i)}_t
\end{equation}
\textbf{Base forecasts} were created by fitting a GP (\ref{subsec:background:forecasting:gpr}) to the observations $\pmb{y}^{(n)}$ and are parameterized by $\{\Theta^{(n)}\}_{n=1}^4$ for each series $\{\pmb{x}^{(n)}\}_{n=1}^4$ and $\Eta$ for $\pmb{u}$:
\begin{equation}
\label{eqn:experiments:data_sets:synthetic_a:base_forecasts}
    \hat{\pmb{x}}^{(n)} \equiv \pmb{x}^{(n)} | \Theta^{(n)} \sim \mathcal{GP}^{(n)} \quad \forall n=\{1,2,3,4\}, \quad \pmb{u} | \Eta \sim \mathcal{GP}
\end{equation}
The GP's are set with a scaled periodic kernel (see~\ref{eqn:background:gpr:periodic_kernel}). The period length hyperparameter $p$ is assumed to be known and fixed at $2\pi$, which is a common case in many applications where the seasonality is known.

\textbf{Reconciliation} for each time step is performed independently, so the reconciler actually gets the marginalised prediction for time step $t$:
$$\hat{x}_t^{(n)} \equiv x_t^{(n)} | \Theta_t^{(n)} \sim \mathcal{N}_t^{(n)} \quad \forall n=\{1,2,3,4\}, \quad u_t | \Eta_t \sim \mathcal{N}_t$$
A visualization of the data set and the base forecasts can be seen in Figure~\ref{fig:experiments:synthetic:data_a:viz}.
\begin{figure}
    \includegraphics[scale=0.5, width=14cm]{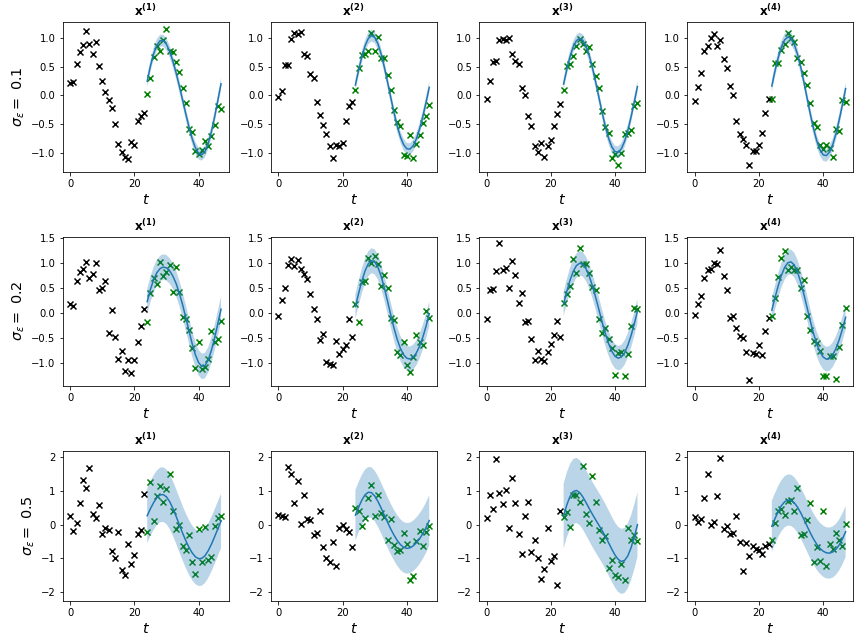}
	\centering
	\caption{Example of a data set and its base forecasts for $\sigma_{\epsilon}=\{0.1,0.2,0.5\}$. Black x's denote training data (observations $\pmb{y}^{(n)}$), green x's denote test data ($\pmb{x}^{(n)}$), blue lines and shaded blue areas denote the mean and 95\% confidence interval of the base forecasts, respectively.}
	\label{fig:experiments:synthetic:data_a:viz}
\end{figure}

\subsection{Results - Experiment A}
\label{subsec:experiments:synthetic:results-a}

\begin{table}[]
    \centering
    \begin{tabular}{|ll|r|r|r|}
    \hline
    {} & {} & {Base Forecast} & {Bayes Paper LG Reconciler} & {ET Reconciler} \\
    \hline
    \multirow[c]{3}{*}{$\sigma_{\epsilon}=0$} & \textit{mean} & -15.73 & \textbf{-15.85} & -15.04 \\
    & \textit{SEM} & 0.00 & 0.00 & 0.00 \\
    & \textit{SD} & 0.00 & 0.00 & 0.00 \\
    \hline
    \multirow[c]{3}{*}{$\sigma_{\epsilon}=0.1$} & \textit{mean} & 2.13 & 2.88 & \textbf{0.26} \\
    & \textit{SEM} & 0.52 & 0.59 & 0.35 \\
    & \textit{SD} & 3.67 & 4.18 & 2.48 \\
    \hline
    \multirow[c]{3}{*}{$\sigma_{\epsilon}=0.2$} & \textit{mean} & 3.77 & 4.37 & \textbf{2.43} \\
    & \textit{SEM} & 0.47 & 0.53 & 0.33 \\
    & \textit{SD} & 3.35 & 3.78 & 2.31 \\
    \hline
    \multirow[c]{3}{*}{$\sigma_{\epsilon}=0.5$} & \textit{mean} & 6.61 & 7.17 & \textbf{5.73} \\
    & \textit{SEM} & 0.37 & 0.42 & 0.27 \\
    & \textit{SD} & 2.59 & 2.95 & 1.91 \\
    \hline
    \multirow[c]{3}{*}{$\sigma_{\epsilon}=1$} & \textit{mean} & 44.38 & 48.55 & \textbf{39.92} \\
    & \textit{SEM} & 21.82 & 24.00 & 19.32 \\
    & \textit{SD} & 154.31 & 169.71 & 136.61 \\
    \hline
    \multirow[c]{3}{*}{$\sigma_{\epsilon}=2$} & \textit{mean} & 139.82 & 163.48 & \textbf{134.89} \\
    & \textit{SEM} & 109.48 & 127.55 & 106.00 \\
    & \textit{SD} & 774.13 & 901.95 & 749.51 \\
    \hline
    \multirow[c]{3}{*}{$\sigma_{\epsilon}=5$} & \textit{mean} & 29.84 & 34.14 & \textbf{29.74} \\
    & \textit{SEM} & 9.64 & 11.68 & 9.43 \\
    & \textit{SD} & 68.16 & 82.60 & 66.71 \\
    \hline
    \end{tabular}
\caption{Results for the experiment on Data Set~\ref{subsec:experiments:synthetic:data-a}.}
\label{tab:experiments:synthetic:results-a}
\end{table}
In Table~\ref{tab:experiments:synthetic:results-a} are the results for Experiment A.

The first takeaway is that the algorithm proposed in~\ref{pro:lg} works, i.e., the \textit{ET Reconciler} improves the \textit{NLPD} score compared to the base forecasts (except for the case where $\sigma_{\epsilon} = 0$ which we will explain in a bit).

Another takeaway from looking at the results is that we can notice that except for when the data was generated with $\sigma_{\epsilon} = 0$, \textit{ET Reconciler} always outperformed \textit{Bayes Paper LG Reconciler}. Here is some intuition why.

The base forecasts were generated using GPs (see~\ref{eqn:experiments:data_sets:synthetic_a:base_forecasts}), and the kernels of a GP have an additive property, i.e., the sum of the four periodic kernels of the bottom time series is also a periodic kernel. So, we expect the mean of the base forecasts for the upper time series to be close to the sum of the means of the base forecasts for the bottom time series, and in the notations of~\ref{subsec:theory:model:linear-model}, we expect  $A \pmb{\mu}_{\theta}$ and $\pmb{\mu}_{\eta}$ to be close.
In both reconciliation methods (\textit{ET Reconciler} and \textit{Bayes Paper LG Reconciler}) there is a mechanism (see $\mu_{LG}$ in Proposition~\ref{pro:lg:mu-lg} and Section 3.1 in~\cite{hier-bayes-paper}) to correct the means of the base forecasts if there is a difference between the mean of the upper series and the means of the bottom series after being passed through the mapping $f(\pmb{x})$, in this data set it is a summation.
Since we used GPs to generate the forecasts there is not a big difference between the means and both reconciliation methods report a similar mean.

So the difference between the reconcilers must come from the covariance.

What happens is that \textit{ET Reconciler} adds to the covariance of the base predictions the difference between the covariance of the prediction for the upper series and the transformed (after the linear mapping) covariance of the prediction for the bottom series (see $\Sigma_{LG}$ in Proposition~\ref{pro:lg:sigma-lg}), wheres in \textit{Bayes Paper LG Reconciler} they always reduce the covariance of the reconciled predictions (see Section 3.1 in~\cite{hier-bayes-paper}).
This property of the reconcilers makes \textit{ET Reconciler}  enlarge its total variance, whereas \textit{Bayes Paper LG Reconciler} reduces its total variance, and as the data set is getting more "noisy", i.e., $\sigma_{\epsilon}$ gets larger (up to a certain level) the reconciler with the larger total variance gets a better \textit{NLPD} score.

As for the experiment where the data set has no noise ($\sigma_{\epsilon} = 0$), what happens is that the base forecasts are quite accurate to begin with and their mean  is accurate, so in order to make the \textit{NLPD} score better one needs to reduce the variance, and that is why this is the only case that \textit{Bayes Paper LG Reconciler} gets a better score. This is of course a very extreme case where the base forecasts are accurate and there is no noise in the data set. Also, the magnitude of improvement of the \textit{NLPD} is very small compared to other settings of $\sigma_{\epsilon}$.

\subsection{Data --- Experiment B}
\label{subsec:experiments:synthetic:data-b}

This setting is the exact same setting as in \ref{subsec:experiments:synthetic:data-a}, with the only difference that the base forecasts themselves, for the bottom level only, were "shifted" for each series by some random noise. So instead of Equation~\ref{eqn:experiments:data_sets:synthetic_a:base_forecasts} we now have:
\begin{equation}
\label{eqn:experiments:data_sets:synthetic_b:base_forecasts}
    \hat{\pmb{x}}^{(n)}_{B} = \hat{\pmb{x}}^{(n)}_{A} + \zeta^{(n)}, \quad \zeta^{(n)} \sim \mathcal{N}(0, 1) \quad \forall n=\{1,2,3,4\}, \quad \pmb{u} | \Eta \sim \mathcal{GP}
\end{equation}
where $\hat{\pmb{x}}^{(n)}_{A}, \hat{\pmb{x}}^{(n)}_{B}$ are the base forecasts in Experiments A and B respectively.
\textbf{Reconciliation} is done the same as in \ref{subsec:experiments:synthetic:data-a}.

A visualization of the data set and the "shifted" base forecasts can be seen in Figure~\ref{fig:experiments:synthetic:data_b:viz}.
\begin{figure}
    \includegraphics[scale=0.5, width=14cm]{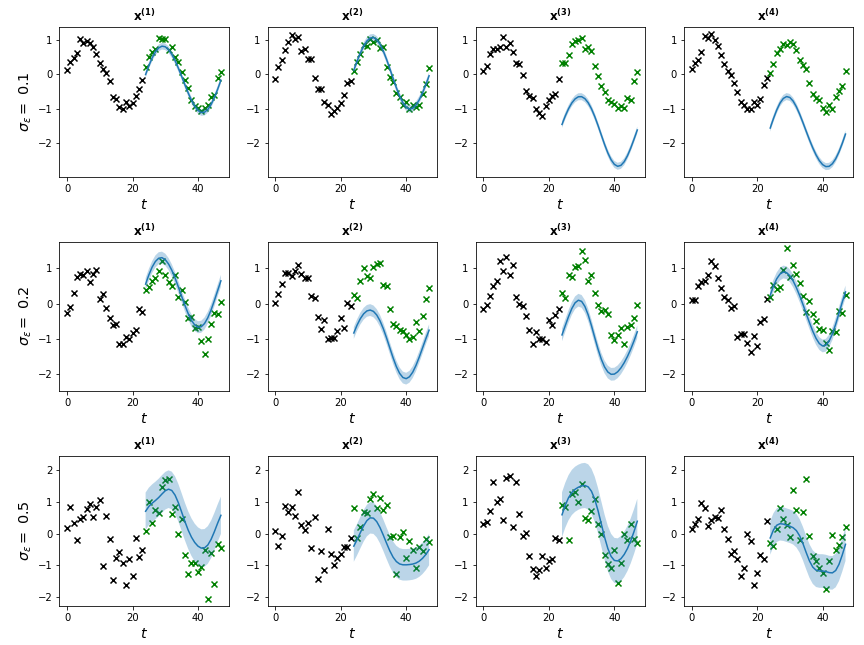}
	\centering
	\caption{Example of a data set and its "shifted" base forecasts for $\sigma_{\epsilon}=\{0.1,0.2,0.5\}$. Black x's denote training data (observations $\pmb{y}^{(n)}$), green x's denote test data ($\pmb{x}^{(n)}$), blue lines and shaded blue areas denote the mean and 95\% confidence interval of the base forecasts, respectively.}
	\label{fig:experiments:synthetic:data_b:viz}
\end{figure}

\subsection{Results - Experiment B}
\label{subsec:experiments:synthetic:results-b}

\begin{table}[]
    \centering
    \begin{tabular}{|ll|r|r|r|}
    \hline
    {} & {} & {Base Forecast} & {Bayes Paper LG Reconciler} & {ET Reconciler} \\
    \hline
    \multirow[c]{3}{*}{$\sigma_{\epsilon}=0$} & \textit{mean} & 15947.89 & 15191.35 & \textbf{12166.44} \\
    & \textit{SEM} & 1727.36 & 1675.84 & 1549.79 \\
    & \textit{SD} & 12214.29 & 11849.97 & 10958.69 \\
    \hline
    \multirow[c]{3}{*}{$\sigma_{\epsilon}=0.1$} & \textit{mean} & 367.63 & 348.50 & \textbf{265.19} \\
    & \textit{SEM} & 46.69 & 44.55 & 38.77 \\
    & \textit{SD} & 330.15 & 315.02 & 274.13 \\
    \hline
    \multirow[c]{3}{*}{$\sigma_{\epsilon}=0.2$} & \textit{mean} & 109.28 & 103.90 & \textbf{79.01} \\
    & \textit{SEM} & 10.64 & 9.83 & 7.75 \\
    & \textit{SD} & 75.21 & 69.51 & 54.83 \\
    \hline
    \multirow[c]{3}{*}{$\sigma_{\epsilon}=0.5$} & \textit{mean} & 20.82 & 20.59 & \textbf{16.75} \\
    & \textit{SEM} & 1.86 & 1.80 & 1.60 \\
    & \textit{SD} & 13.15 & 12.70 & 11.30 \\
    \hline
    \multirow[c]{3}{*}{$\sigma_{\epsilon}=1$} & \textit{mean} & 52.74 & 58.21 & \textbf{43.30} \\
    & \textit{SEM} & 23.05 & 27.23 & 18.19 \\
    & \textit{SD} & 163.00 & 192.53 & 128.64 \\
    \hline
    \multirow[c]{3}{*}{$\sigma_{\epsilon}=2$} & \textit{mean} & 164.91 & 177.92 & \textbf{153.76} \\
    & \textit{SEM} & 102.91 & 112.67 & 98.58 \\
    & \textit{SD} & 727.67 & 796.71 & 697.07 \\
    \hline
    \multirow[c]{3}{*}{$\sigma_{\epsilon}=5$} & \textit{mean} & 107.73 & 119.53 & \textbf{106.10} \\
    & \textit{SEM} & 41.73 & 47.53 & 42.00 \\
    & \textit{SD} & 295.09 & 336.06 & 296.97 \\
    \hline
    \end{tabular}
\caption{Results for the experiment on Data Set~\ref{subsec:experiments:synthetic:data-b}.}
\label{tab:experiments:synthetic:results-b}
\end{table}

In Table~\ref{tab:experiments:synthetic:results-b} are the results for Experiment B.

Here we can notice that again for all settings of $\sigma_{\epsilon}$ \textit{ET Reconciler} outperforms \textit{Bayes Paper LG Reconciler},  even when $\sigma_{\epsilon} = 0$. This happens here because for all settings the reconcilers must change the mean of the base forecasts, as well as the covariance, so the fact that for $\sigma_{\epsilon} = 0$ there is no noise does not affect the improvement of the \textit{NLPD} compared to the base forecasts.
As in the previous experimental results (Table~\ref{tab:experiments:synthetic:results-a}), the \textit{NLPD} improves up to a certain level of noise in the data. Also  notice that here even for a very large noise ($\sigma_{\epsilon} = 5$), \textit{ET Reconciler} improves the \textit{NLPD} whereas \textit{Bayes Paper LG Reconciler} does not.

\section{Real Data Set}
\label{sec:experiments:real}

\subsection{Data}
\label{subsec:experiments:real:data}
\begin{figure}
	\centering
    \includegraphics[scale=0.5, width=14cm]{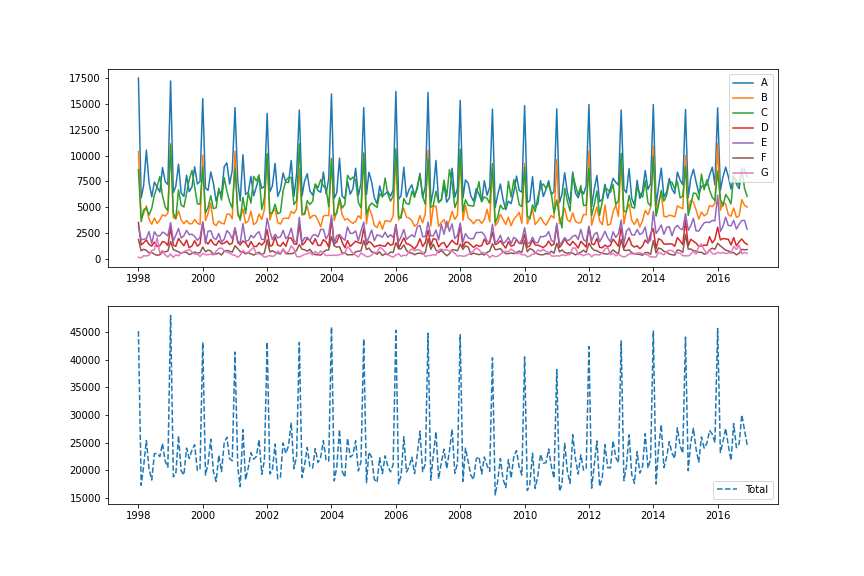}
	\caption{The first level of the hierarchy of the dataset. Top graph: the number of visitors by month to each of the seven states of Australia. Bottom graph: The sum of the series in the top graph.}
	\label{fig:experiments:real:data:viz}
\end{figure}
The real data set we tested on is a data set consisting of the number of tourists who entered Australia on a monthly basis, by their reason of entrance (Holiday, Visitor, Business, Other), between 1998 to 2016. On the bottom level the data set contains 304 time series referring to Australia's 76 regions and 4 reasons of entrance. The data is then aggregated into 27 macro zones and then into 7 states on the upper level. We chose to test the 7 series of the total tourists entering  each state. A visualization of the data set is shown in Figure~\ref{fig:experiments:real:data:viz}.

\textbf{Base forecasts} were created using the same methodology of the synthetic data set (see~\ref{subsec:experiments:synthetic:data-a}), with the only difference that the period length hyperparameter $p$ was set to $12$ and that we added an RBF kernel (see~\ref{eqn:background:gpr:rbf_kernel}).

\textbf{Reconciliation} here is also performed independently as in \ref{subsec:experiments:synthetic:data-a}, and got the same marginalised prediction for time step $t$.
The main change here regarding train and test data, is that we are splitting the data set into training and test at a different point in each simulation we perform, i.e., we perform 50 simulations, each with 60 time steps (5 years) as training data and 24 time steps (2 years) as test data, starting from the first period.

\subsection{Results}
\label{subsec:experiments:real:results}

\begin{table}[]
    \centering
    \begin{tabular}{|l|r|r|r|}
    \hline
    {} & {\textit{mean}} & {\textit{SEM}} & 
    {\textit{SD}} \\
    \hline
    Base Forecast & 60.56 & 0.41 & 2.88 \\
    \hline
    Bayes Paper LG Reconciler & 61.19 & 0.41 & 2.87 \\
    \hline
    ET Reconciler & \textbf{59.35} & 0.42 & 2.99 \\
    \hline
    \end{tabular}
\caption{Results for the experiment on Data Set of Section~\ref{subsec:experiments:real:data}.}
\label{tab:experiments:real:results}
\end{table}

 Table~\ref{tab:experiments:real:results} depicts the results for the experiment on the real data set.
Here we can see that our proposed method works for the real data set as well, i.e., the NLPD of the reconciled forecasts decreases with respect to the base forecast's NLPD.

In this dataset, the observations for the bottom series are quite noisy, while for the upper series they are less noisy (see Figure~\ref{fig:experiments:real:data:viz}).
The base forecasts were generated with a GP with periodic and RBF kernels, which generated quite good base forecasts in the mean sense, i.e., they were quite accurate in their mean but not in their variance.
Due to the additive property of the GPs, the mean of the sum of the base forecasts for the bottom level was close to the mean of the upper level series, but the base forecasts for the bottom level were too ``certain'', i.e., had relatively small variance, that had to be taken care of.
\textit{Bayes Paper LG Reconciler}, as mentioned in Section~\ref{subsec:experiments:synthetic:results-a}, always decreased the total variance of the reconciled forecasts, while our \textit{ET Reconciler} did not.
So, \textit{ET Reconciler} made the reconciled forecasts less "certain", while \textit{Bayes Paper LG Reconciler} made the base forecasts more "certain", i.e., decreasing and increasing total variance respectively.
Our metric is a probabilistic metric, i.e., it takes into account the density of the prediction, and that property along with the fact that the base forecasts were too ``certain'' but quite accurate in the mean sense, made \textit{ET Reconciler} improve the NLPD and \textit{Bayes Paper LG Reconciler} worsen the NLPD, compared to that of the base forecasts.

\chapter{Conclusion}
\label{ch:discussion}

Time series forecasting is a key task in the ability to make informed decisions under uncertainty.
There are many cases where the number of series we need to predict is too large to fit in a single model.
In that case, a distributed approach combined with a way to utilize the knowledge of the hierarchical structure of these series, allows to get better predictions in a reasonable runtime.
The process is to learn base forecasts for each series independently and for some summary statistics series based on the hierarchy structure, and pass those base forecasts through a reconciliation algorithm to improve each forecast.
In this work we tackled the reconciliation problem, with the assumption that the base forecasts are given to us.

Reconciliation is an active field of research, but it is worth noting that works in this field, including the ones cited in Section \ref{sec:related_work:forecast_reconciliation}, are trying to solve a somewhat different problem --- assuring coherency of the base forecasts, i.e., that the summary statistics series forecasts are aligned with both the forecasts for each series and the hierarchical structure.
As stated before, we tried to solve a somewhat different problem of reconciliation in a distributed setting.

In this work we presented a method for Bayesian forecast reconciliation of hierarchical time series.
We defined the reconciliation problem as a generative model, and derived the posterior distribution of the time series of interest given the historical observations that is given to us through the individual forecasts of the series and the forecast for the summary statistics.
We highlighted the assumptions in our case about the accuracy of the base forecasts for the summary statistics, and about the prior for the base forecasts for the series of interest.
We finally addressed the specific case of linear Gaussian reconciliation --- a case where the base forecasts are Gaussian distributed and the summary statistics have a linear mapping.
We showed, by deriving the formulas, that reconciliation for this case can be achieved in closed form of a multivariate Gaussian distribution.

A limitation of the proposed method is that we assume that although it is intractable to forecast the series in a single model where the number of series is very large it is possible to forecast their summary statistics in a single model.
This is an assumption that we have to make for the validity of the proposed method.

We evaluated our linear Gaussian reconciliation method on both a synthetic data set and a real data set.
We saw that our method improved the base forecasts on the metric described in Section~\ref{sec:experiments:metrics}.
We then compared our method to the one proposed in~\cite{hier-bayes-paper} on both data sets by the same metric.
The results showed that our method outperformed theirs on both the synthetic data set and the real data set.
This difference might come from the fact that the data was noisy and the adjustments that had to be made where in the variance of the base forecasts and not on the mean of the forecasts, and the method proposed in~\cite{hier-bayes-paper} does not do well in that case.

Our experiments were done on data sets that have a small number of series while our proposed method was set to work on a setting with a very large number of series, so, future work should definitely include testing the methods proposed in this work on a large scale data set.
Another direction is reconciliation through time --- we tested our method on reconciliation of each time step independently, and future work could try to test this as well.

\begin{appendices}

\end{appendices}

\bibliography{main.bbl}

\clearpage

\end{document}